\documentclass[letterpaper, 10 pt, conference]{ieeeconf}
\IEEEoverridecommandlockouts                              
\usepackage{array}
\usepackage{units}

\usepackage{multirow}
\usepackage[pdftex]{graphicx} 
\usepackage{subcaption}
\usepackage[cmex10]{amsmath}
\usepackage{amssymb}
\usepackage{amsfonts}
\usepackage{supertabular}
\usepackage{times}
\usepackage{floatrow}
\usepackage{float}
\usepackage{mathtools}
\usepackage{textcomp}
\usepackage{algpseudocode}
\usepackage{algorithm}
\usepackage{algorithmicx}
\usepackage{adjustbox}
\usepackage[justification=centering, font=small]{caption} 
\usepackage{epstopdf}
\usepackage{tabularx}
\usepackage{nccmath}
\usepackage{mathtools}
\usepackage{cite}

\newtheorem{thm}{Theorem}

\newtheorem{assumption}{Assumption}
\newtheorem{remark}{Remark}

\newtheorem{prop}{Proposition}

\newfloatcommand{capbtabbox}{table}[][\FBwidth]

\newcommand{\vect}[1]{\boldsymbol{#1}}

\graphicspath{{./Figures/}}

	\title{\LARGE \bf Learning to Control Using Image Feedback}	
    \author{R. Krishnan$^1$, N. Vignesh$^2$, and S. Jagannathan$^3$
\thanks{R. Krishnan is with the Mathematics and. Computer Science Division, Argonne National Laboratory, IL, USA. N. Vignesh is with the AI Institute and the Dept. of Computer Science and Engineering, University of South Carolina, Columbia, SC, USA. S. Jagannathan is with the Dept. of Electrical and Computer Engineering, Missouri University of Science and Technology, Rolla, MO, USA. kraghavan@anl.gov, vignar@sc.edu, and sarangap@mst.edu}}

\begin{document}
	\maketitle
\begin{abstract}
Learning to control complex systems using non-traditional feedback, e.g., in the form of snapshot images, is an important task encountered in diverse domains such as robotics, neuroscience, and biology (cellular systems). In this paper, we present a two neural-network (NN)-based feedback control framework to design control policies for systems that generate feedback in the form of images. In particular, we develop a deep $Q$-network (DQN)-driven learning control strategy to synthesize a sequence of control inputs from snapshot images that encode the information pertaining to the current state and control action of the system. Further, to train the networks we employ a direct error-driven learning (EDL) approach that utilizes a set of linear transformations of the NN training error to update the NN weights in each layer. We verify the efficacy of the proposed control strategy using numerical examples. 
\end{abstract}
\section{Introduction}\label{sec:Introduction}
Learning to control complex systems using feedback data is an active research topic in systems and control theory. Approaches ranging from traditional adaptive control \cite{narendra2012stable} to more modern techniques such as adaptive dynamic programming (ADP) \cite{bertsekas2015value,lewis2012reinforcement} /reinforcement learning (RL) \cite{sutton2018reinforcement} that employ artificial neural networks (ANNs) have been proposed to steer systems without relying on complete model information. These control approaches use either the state or output feedback data generated by the system to learn policies for steering the system as desired. In many emerging systems such as robotic swarms, neural systems, and cellular populations, the feedback information is only available in the form of snapshot images~\cite{yu2020learning,mnih2015human}. Consequently, commonly used learning-based control strategies including ADP/RL techniques have to be upgraded to incorporate feedback data in the form of images and synthesize control policies for steering the system as desired, especially when no desired trajectory or (explicit) state measurements are available. 

In this context, recent advancements in the deep $Q$ learning, spurred by \cite{mnih2015human}, have demonstrated the capacity of RL models employing deep neural networks (DNNs) to perform such complex decision making tasks. In an RL model, an agent or a team of agents learn to achieve the desired objective by learning the best sequence of actions from any situation/state \cite{ref_extending_Iwata}  using experiences (data) collected over time. The validity of such an RL model is typically quantified as a function of the temporal difference (TD) error \cite{sutton2018reinforcement}, which the agent seeks to reduce during learning.
Specifically, the deep $Q$-network (DQN) introduced in \cite{mnih2015human} was the first approach to design a deep RL model in which the TD error was used to learn the weights of a $Q-$network, a parameterized DNN, by mimicking a supervised learning procedure, where a clone of the $Q-$ network~(referred as target network) was used to generate the target values. In \cite{mnih2015human}, the target network was periodically updated with the latest weights of the $Q-$network, providing a stable training regime that mitigated short-term oscillations induced by the moving targets in a standard RL algorithm. \emph{In this context, the problem of designing a learning controller boils down to learning a $Q$-function with a parameterised DNN.}


In this paper, we introduce a DNN-based learning strategy to synthesize feedback policies for steering a system generating feedback in the form of images. To this end, we adopt and expand the DQN framework in \cite{mnih2015human} and propose a dual NN-driven exploratory learning approach for improving the learning time and efficiency of the DQN. In the proposed approach, two DNNs, independently initialized, take the role of the target and the $Q-$network~(actor) alternatively in successive iterations. Furthermore, we propose training rules to update the weights of the DNN using a direct error-driven learning~(EDL) strategy, which introduces a set of linear transformations that help compute training signals corresponding to each layer. We refer to this dual NN driven alternative update mechanism as the \emph{cooperative update strategy} (coop). We demonstrate the convergence of EDL analytically, and show faster convergence for EDL in comparison to standard DQN, making it a feasible candidate for developing online control strategies in emerging applications.

\section{Related Works}
There has been a growing interest in combining deep learning with RL and control to enable agents navigating complex and high-dimensional state-space to perform control tasks with non-traditional sensory data. For instance, a neural fitted Q-learning (NFQ) was introduced in \cite{NFQ_05} using resillient backpropagation algorithm to train the $Q$-network \cite{lange_10_deep}. More recently, the work in \cite{mnih2015human} applied DRL to learn the action-value function through the stochastic gradient-based update rule. An experience replay (ER) strategy~\cite{ref_self_improve_Lin, ref_batch_rein_learn_Kalyanakrishnan, ref_Qlearn_exper_replay_Pieters} was also employed in \cite{mnih2015human} to improve efficiency and reduce short term oscillations that are prevalent in RL algorithms~\cite{ref_synth_Leonetti, ref_general_fact_domain_Hester}. However, with the ER strategy, the memory required to store the experiences may be limited or expensive in many applications~\cite{ref_backward_Qlearning_Wang, ref_Qlearn_exper_replay_Pieters}. 
Moreover, despite several successful and intelligent sampling strategies to improve implementation efficiency, many RL approaches are still sample-inefficient \cite{jiang2020model}.

On the other hand, an efficient weight update rule is necessary~\cite{ref_fmrq_multi_Zhang, ref_moo_mdp_Silva} to avoid slow learning due to vanishing gradient problem~\cite{Ragh2017error,hardt2016train} and lack of efficient exploration common with back-propagation/stochastic gradient descent (SGD) algorithm~\cite{mnih2015human}. More recently,  a direct error driven learning (EDL) approach was introduced in \cite{Ragh2017error} to address the deficiencies of SGD. EDL is designed to train DNNs \cite{Ragh2017error} within the supervised learning setup where explicit targets are defined at each layer by performing a linear transformation of the error. Such a transformation allowed the EDL mechanism to avoid the vanishing gradient problem by side-stepping the layer-wise activation functions that introduce this issue. However, the EDL approach has not been utilized in DRL applications and a theoretical analysis of the EDL algorithm is a missing component in \cite{Ragh2017error}. 

In this work, we introduce a dual NN driven exploratory learning approach. In contrast with the existing works that incorporate two NNs \cite{doubleDQN, dueling}, we use two independently initialized NNs to cooperatively learn from each others' experience, mitigating potential bias due to initial experiences or network initialization, and thus improving sample efficiency. Furthermore, EDL provided an inherent exploration strategy that improved convergence with enhanced exploration of the parameter space in contrast with typical SGD~\cite{Ragh2017error,hardt2016train}.

\section{Preliminaries}\label{sec: background}

We consider applications where a sequence of images or snapshots are obtained from the system and must be used to choose a control input to the system from a discrete set of feasible actions. Therefore, consider an RL problem, where an agent interacts with an environment~(simulates the system) through a sequence of control actions and rewards. With each control action $\vect{u}(k)$ taken by the agent, the internal states of the environment change. Let $\vect{x}(k) \sim p( \mathcal{S} )$ be the state of the environment at the $k^{th}$ sampling instant and $p( \mathcal{S} )$ be the distribution on the states with the state space denoted by $\mathcal{S}$. Define $\mathcal{A}$ to denote the set of feasible actions,  $\mathcal{R}$ to denote the set of feasible rewards, and $\Omega$ to denote a measurable space.  The action $u(k)$ is, in general, based on a control policy $\pi:\mathcal{S}\to \mathcal{A}$, and as the agent takes an action, the internal state of the environment changes from $x(k)$ to $x(k+1)$ such that  $\vect{x}(k+1) \sim T(\vect{x}(k), \vect{u}(k), \vect{\omega}(k)),$ where  ${T}:\mathcal{S}\times\mathcal{A}\times \Omega\to\mathcal{S}$ is the state transition function with $\vect{\omega}(k) \in \Omega$ being the internal fluctuations, e.g., noise. In this process, the agent receives an external scalar reward $r(k+1) \in  \mathcal{R}$ from the environment. The tuple $( x(k), u(k),  x(k+1), r(k+1))$ is collected as an experience by the RL agent to learn the $Q$-function during its interaction with the environment at time instant $k$.

In our work, we employ NNs to approximate the unknown $Q-$function. To this end, we define a NN with $d$ layers, denoted with an ideal parametric map, $y(\vect{z}; \vect{\theta})$, where $\vect{z}$ is the input and $\vect{\theta} = [\vect{W}^{(1)} \cdots \vect{W}^{(d)}]$ is the collection of all the ideal weights for this approximation. Each element in $\vect{\theta}$, i.e. $\vect{W}^{(i)}, i = 1, \cdots d$ denotes the weights of the DNN at layer $i.$ Let $\varepsilon$ denote bounded approximation error~\cite{paper3_approximation} incurred  by the ideal map and $f^{(i)}$, for $i = 1, \ldots, d$, denote the layer-wise activation functions corresponding to the $d$ layers. To denote the estimated NN, we introduce the following set of notations. The estimated weights are denoted by $\hat{\vect{\theta}}  = [\hat{\vect{W}}^{(1)} \cdots \hat{\vect{W}}^{(d)}]$, and the estimate of the function $y(\vect{z},\vect{\theta})$ is given by $\hat{y}( \vect{z}; \hat{\vect{\theta}} ),$ where the symbol $\hat{(\cdot)}$ is used to denote estimated quantities. 

Our network is designed to take as input, the images, and provide as output the $Q-$function value corresponding to each feasible action to the system as in \cite{mnih2015human}.  We will then employ a greedy policy and choose the control action that obtains the largest $Q-$function value. The learning objective is to learn the optimal $Q-$function. We denote the reward function as $r(\vect{x}(k), \vect{u}(k))$, a function of both the state and the action or control inputs. For a given (state, action) pair, $(\vect{x}(k), \vect{u}(k))$ and an action policy $\pi,$ the $Q$-function given as $  Q^{\pi}(\vect{x}(k),\vect{u}(k)) =\sum_{k = 0}^{\infty} \gamma^{k} r(\vect{x}(k), \vect{u}(k)),$ where $\gamma \in (0,1)$ is the discount factor. Additionally, we use $\mathbb{N}$ and $\mathbb{R}$ for denoting the set of natural and real numbers, respectively. Furthermore, we use $\|.\|$ to denote Euclidean/ Frobenius norm for vectors/matrices.  

\section{Methods}
\label{sec:Methods}
In the following, we first introduce the optimization problem and the overall learning architecture. We then present the details of the coop learning scheme.
\subsection{RL Approximation of $Q$-values} 
In our problem setup, feedback from the system is received by the RL agent in the form of images. Then, the RL agent must learn a map between the image inputs and the $Q$ values corresponding to the set of available actions. To update the $Q$ values, we use the Bellman's equation \cite{sutton2018reinforcement} provided as
    \begin{equation}
    	\begin{aligned}\label{eq:bellman}
    	   Q^{*}(\vect{x}(k), \vect{u}(k))  = max_{u(k+1)} (r(\vect{x}(k), \vect{u}(k)) \nonumber \\ + \gamma^{k} (Q^{*}(\vect{x}(k+1), \vect{u}(k+1)) ),
    	 \end{aligned}
    \end{equation}
where $Q^*$ refers to the optimal $Q$-function with $\rho(\mathcal{A})$ denoting the distribution over $\mathcal{A}$. A greedy policy is given as $\vect{u}^{*}(\vect{x}(k))  = argmax_{\vect{u}(k)} ( Q^{\rho}(\vect{x}(k), \vect{u}(k)) )$ and $\rho$ refers to  $\rho(\mathcal{A})$.  In this work, we approximate the $Q$-function with a generic parametric map, $y(\vect{x}(k), \vect{u}(k) ; \vect{\theta})$ with an ideal set of parameters $\vect{\theta}$ such that ${Q}^*(\vect{x}(k), \vect{u}(k)) = y(\vect{x}(k), \vect{u}(k) ; \vect{\theta})$~(hitherto denoted as $y$). We want to find $\hat{y}$ with $\hat{\vect{\theta}}(k)$ for $y$ such that $y$ is approximated by DNNs. The cost is defined as the expected value over all the states and all the actions such that $ \mathbb{E}_{\vect{x}(k) \sim p(\mathcal{S}), \vect{u}(k) \sim \rho(\mathcal{A}) }[ J(\hat{\vect{\theta}}(k) )],$ where
    \begin{align}
    J(\hat{\vect{\theta}}(k) ) =  \big[ \underbrace{\frac{1}{2}(y(\vect{x}(k), \vect{u}(k) ; \vect{\theta})- \hat{y}(\vect{x}(k), \vect{u}(k) ; \hat{\vect{\theta}}))^{2}}_{ \text{Empirical Cost}} \nonumber \\ + \underbrace{\varepsilon^{2} + \varepsilon^{T}( y(\vect{x}(k), \vect{u}(k) ; \vect{\theta}) - \hat{y}(\vect{x}(k), \vect{u}(k) ; \hat{\vect{\theta}}))}_{\text{Approximation Error Cost}} \big],
  \end{align}
with the target $y(\vect{x}(k), \vect{u}(k) ; \vect{\theta}) = r(\vect{x}(k), \vect{u}(k))  + \gamma^{k}{\hat{Q}^{*}(\vect{x}(k+1), \vect{u}(k+1))}$, coming from the Bellman equation with $\hat{Q}^{*}$ denoting the approximation of the optimal $Q$-value obtained by extrapolating the current $Q$-function estimate. Only the empirical cost, the first term, can be minimized as the other two terms depend on the choice of the NN, the available data, and the problem complexity, which is typically compensated via the use of regularization approaches \cite{hardt2016train,lewis1998neural}. Therefore, we minimize the empirical cost as 
    \begin{equation}
    	\begin{aligned}
    	   \hat{\vect{\theta}}(k) = arg~min_{\bar{\vect{\theta}}(k)} \mathbb{E}_{\vect{x}(k) \in \mathcal{S}, \vect{u}(k) \sim \rho(\mathcal{A}) } \\ \big[ J_E(\vect{x}(k), \vect{u}(k),\hat{y}(\vect{x}(k), \vect{u}(k) ; \hat{\vect{\theta}}))\big], & 
    	\end{aligned}
    	 \label{eq_main_opt}
    \end{equation}
with $J_E(\vect{x}(k), \vect{u}(k) , \hat{y}( \vect{x}(k), \vect{u}(k); \hat{\vect{\theta}}(k)) )= \frac{1}{2} \big[r_{k} + \gamma^{k}\hat{Q}_{k+1}^{*} - \hat{y}_{k} \big]^2 = \frac{1}{2} \big[ \epsilon_{k} \big]^2$, with $r(\vect{x}(k), \vect{u}(k))$ denoted by $r_{k}$, {$\hat{Q}^{*}(\vect{x}(k+1), \vect{u}(k+1))$} by $\hat{Q}_{k+1}^{*}$ and $\hat{y}(\vect{x}(k), \vect{u}(k) ; \hat{\vect{\theta}})$ by $\hat{y}_{k}$. The empirical cost~(denoted hitherto as $J_E$ for notational simplicity) which we seek to minimize is the squared temporal difference (TD) error ($\epsilon_{k}$). 

In the following, we describe a TD error driven approach to update the NN weights. Specifically, we sketch the error driven learning rules first, and then, provide the algorithm for the proposed strategy synthesize a sequence of controls.
\subsection{Error driven TD learning}
In the earlier section, the problem of obtaining a control policy was formulated as an optimization problem where we attempt to learn the optimal $Q$ function through NNs by minimizing $J_E$. To learn the parameters of the optimal $Q$ function, consider the update rule as
\begin{align} \label{eq:weight_update_gen}
	\hat{\vect{\theta}} (k+1) = \hat{\vect{\theta}} (k) + \alpha \vect{\Delta}(k), 
\end{align}
where $\alpha >0$ is the learning rate, $k$ is the sampling instant, and $\vect{\Delta}(k)$ is the parameter adjustment. To simplify notations, we switch to a subscript notation instead of expressing the sampling instants in parenthesis so that \eqref{eq:weight_update_gen} is rewritten as $\hat{\vect{\theta}}_{k+1} = \hat{\vect{\theta}}_{k} + \alpha \vect{\Delta}_{k}$. From here on, for brevity, the notation of the expected value operator is suppressed, and we refer to the expected value of the cost as just the cost. Next, add a regularization term to the cost function to get 
\begin{align}
    {{H}(\hat{\vect{\theta}}_k)}= \big[ J_{E}(\hat{\vect{\theta}}_{k})  +   \sum_{i = 1}^{d} \lambda^{(i)} R_{k}(\hat{\vect{W}}^{(i)}_{k})\big],
    \label{eq:orignal}
\end{align}
where $\lambda^{(i)} > 0$ is the decay coefficient and  $R_{k}$ is the regularization function applied on $\hat{\vect{W}}^{(i)}_{k}$ with $J_E(\hat{\vect{\theta}}_{k})$ being used in place of $J_E(\vect{x}(k), \vect{u}(k); \hat{\vect{\theta}}(k)) $ for notational simplicity. The most common way of obtaining the control in this context is to backpropagate the TD error and we write the weight adjustment at each $k$ as
 \begin{align} \label{eq11}     \vect{\Delta}^{(i)}_{k}  = -\big[ \nabla_{\hat{\vect{W}}^{(i)}_{k}}   H(\hat{\vect{\theta}}_{k})  \big]  = -\big[ {\vect{\delta}}^{(i)}_{k}  + \lambda^{(i)}  \nabla_{\hat{\vect{W}}^{(i)}_{k}} R(\hat{\vect{W}}^{(i)}_{k}) \big],\end{align}   
where the term $\nabla_{\hat{\vect{W}}^{(i)}_k}(\cdot)$ denotes the gradient of $(\cdot)$ with respect to the NN weight  $\hat{\vect{W}}^{(i)}_{k}$, the second component depends on the choice of $R(\hat{\vect{W}}^{(i)}_{k})$, and $\vect{\delta}^{(i)}_{k}$  is obtained by applying the chain rule to compute the gradient of the cost with respect to the NN weights. A generalized expression for ${\vect{\delta}}^{(i)}_{k} $ can be derived using the chain rule \cite{Ragh2017error,lecun2015deep} as
\begin{align} \label{eq2_Grad}
      {\vect{\delta}}^{(i)}_{k} = f^{(i-1)}(\vect{x}) \epsilon_{k} \vect{\mathcal{T}}^{(i)} \vect{I}^{(i)}.
\end{align} 
where $ \prod_{j = d}^{i+1}( diag(\nabla f^{(j)}(\vect{x})) \hat{\vect{W}}^{(j)}) diag(\nabla f^{(i)}(\vect{x})) $ is denoted as $\vect{\mathcal{T}}^{(i)}.$ Observe that the error $\epsilon_{k}$  has to propagate through transformation $\vect{\mathcal{T}}^{(i)}$ to impact learning. For a fixed $\vect{x},$ $\vect{\mathcal{T}}^{(i)}$ is a linear approximation (gradient) in the neighborhood of the weights. Since, this approximation is a matrix, the singular vectors of $\vect{\mathcal{T}}^{(i)}$ dictate the directions of learning. The magnitude of this learning depends on the singular values of $\vect{\mathcal{T}}^{(i)}$. Each diagonal element in $\vect{\mathcal{T}}^{(i)}$ is the product of the derivative of the layer-wise activation functions. Therefore, in specific cases, when the derivative of the activation function provides singular values between 0 and 1, the singular values of $\vect{\mathcal{T}}^{(i)}$ would tend towards zero as the number of layers in the deep NN increase. The issue is otherwise known as the vanishing gradients issue~(see \cite{pascanu2013difficulty} and the references therein for details).

Motivated by the success of direct error-driven learning in classification problems \cite{Ragh2017error}, we introduce a user-defined feedback matrix denoted as  $\vect{B}^{(i)}_{k}$ (a hyperparameter), with $\vect{B}^{(i)}_{k} (\vect{B}^{(i)}_{k})^{T}$ being positive definite, to take the role of $\vect{\mathcal{T}}^{(i)}_{k}$ in \eqref{eq2_Grad}. The new feedback, denoted $\vect{\sigma}_{k}^{(i)}$, is then given as
\begin{align} \label{eq_EDL}
    \vect{\sigma}_{k}^{(i)}  = [f^{(i-1)}(\vect{x})]  \epsilon_k \vect{B}^{(i)}_{k}.
\end{align}
With this definition of feedback, the new layer-wise cost  function, $\mathcal{H}^{(i)}(\hat{\vect{W}}^{(i)})$, is defined to be
 \begin{align} \label{eq_cost}
    \mathcal{H}^{(i)}(\hat{\vect{W}}^{(i)}) = \frac{1}{2} \big[tr( ( \vect{\sigma}_{k}^{(i)} )^{T} \vect{P} {\hat{\vect{W}}^{(i)}}_{k} )  + \lambda^{(i)} R({\vect{W}^{(i)}_{k}}) \big], 
\end{align} 
where $tr(.)$ is the trace operator and $\vect{P}$ is  a positive definite symmetric matrix of choice. Note that for each layer $i$, $\vect{\sigma}^{(i)}_{k}$ can be understood as the feedback provided by the overall cost $J(\vect{\theta})$ towards controlling the layer $i$ of the DNN. {Therefore, the layer-wise cost can be interpreted as the minimization of the correlation between $\vect{\sigma}^{(i)}_{k}$ and  $\hat{\vect{W}}^{(i)}_{k}$ under the constraint that $\|\hat{\vect{W}}^{(i)}_{k}\|$ is bounded (due to the regularization).} Finally, the overall cost, $ \mathcal{H}(\hat{\vect{\theta}}(k))$, may be written as the sum of layer-wise costs as 

\begin{align}    \label{total_cost}
    \mathcal{H}(\hat{\vect{\theta}}(k)) = \sum_{i=1}^{d}  \mathcal{H}^{(i)}( \hat{\vect{W}}_{k}^{(i)}).
 \end{align}

With the cost $ \mathcal{H}(\hat{\vect{\theta}}_{k})$ defined in Eq. \eqref{total_cost}, 
the weight updates for each layer are defined as $\vect{\Delta}^{(i)}_{k} = - \nabla_{\hat{\vect{W}}^{(i)}_{k}} \mathcal{H}(\vect{x}(k), \vect{u}(k); \hat{\vect{\theta}}(k))$, which results in an update rule for each layer similar to Eq. \eqref{eq2_Grad} with $\vect{\delta}^{(i)}_{k}$ replaced by $\vect{\sigma}_k^{(i)}$, and defined as in Eq. \eqref{eq_EDL}.
 \begin{remark}\label{rem:feedbak_matrix}
 The direct error driven learning (EDL)-based weight updates proposed in \cite{Ragh2017error} can be interpreted as an exploratory update rule, as the descent direction in the weight updates are assigned through $B^{(i)}_k$ matrix. To define the matrix $B^{(i)}_k$, we decompose it as $\vect{B}^{(i)}_{k}. = \vect{U}^{(i)}_{k}(\Sigma^{(i)}_{k}+s \times I^{(i)})\vect{V}^{(i)}_{k},$ where $\vect{U}^{(i)}_{k} \Sigma^{(i)}_{k} \vect{V}^{(i)}_{k}$ is the SVD of $\vect{\mathcal{T}}^{(i)}_{k}$ for a fixed $\vect{x}$ with $I^{(i)}$ being an identity matrix of appropriate dimensions and $s$ is a chosen perturbation.
  \end{remark}
 In our coop learning architecture, we use two NNs, and train both of these networks using the EDL update rules just described. Before presenting the coop strategy, we first analyze the EDL rules in the next section. 
\subsection{Analysis of EDL for training DNN}\label{sec: analysis}
In this section, we analyze different components of the EDL-based learning methodology. All the proofs for the theoretical claims made in this section are given in the supplementary information available online at \cite{Ragh2021supp}. In our analyses, we make the following assumptions. 
\begin{assumption}
    The cost function is a continuous function with continuous derivatives, i.e., $J_{E} \in C^1$. For all $\vect{x}\in \mathcal{S}, \vect{u}\in \mathcal{A}$, and $\hat{\vect{\theta}} \in \Theta$, the empirical cost and its gradient are bounded, i.e., $J_{E}(.,.,\hat{\vect{\theta}}) \leq L$  and $\frac{\partial J_E}{\partial \hat{\vect{\theta}}}\le M$. There exists a positive constant $\vect{\theta}_B$ such that for any $ \hat{\vect{\theta}}, \vect{\theta}$ in the parameter space $\Theta$, we have $max( \|\hat{\vect{\theta}}_k\|,  \|\vect{\theta}\|) \leq \vect{\theta}_B$. The activation functions of the NNs are chosen such that for any $x\in \mathbb{R}$, $\|f^{(i)}(x)\|\le \sqrt{\eta}$, where $\eta$ is the number of neurons in the $i^{th}$ layer. Finally, the perturbations $s \sim \mathcal{N}(0,1)$, where $\mathcal{N}$ denotes the normal distribution.
\label{ass1}
\end{assumption}

\begin{prop}\label{prop:cost_error}
Under the Assumption 1 and with the feedback matrix $\vect{B}^{(i)}_{k}$ defined as in Remark \ref{rem:feedbak_matrix}, the difference between  original cost in \eqref{eq:orignal} and the new cost in \eqref{total_cost} satisfies the bound given by
\begin{align}
  \| H( \hat{\vect{\theta}}_{k} )-\mathcal{H}( \hat{\vect{\theta}}_{k} )  \| \leq  d \frac{|s|}{2} \|\epsilon_{k}\| W_B \sqrt{\eta} + \|\xi\|,
\end{align} 
when $\vect{P}$ in \eqref{eq_cost} satisfies $\|\vect{P}\| \leq 1$, where $W_B$ and $\xi$ are positive constants. 
\end{prop}
\begin{proof}
Consider the quantity $H( \hat{\vect{\theta}}_{k} ) - \mathcal{H}( \hat{\vect{\theta}}_{k} )$ with the definitions of $H( \hat{\vect{\theta}}_{k} )$ given in Lemma 1 and $ \mathcal{H}( \hat{\vect{\theta}}_{k} )$ in \eqref{eq15} to get 
\begin{align} 
  H( \hat{\vect{\theta}}_{k} ) - \mathcal{H}( \hat{\vect{\theta}}_{k} ) &=&  \sum_{i=1}^{d}\frac{1}{2} \big[tr( ( \vect{\delta}_{k}^{(i)} )^{T} \vect{P} {\hat{\vect{W}}^{(i)}}_{k} ) + \lambda^{(i)} R({\vect{W}^{(i)}_{k}}) \nonumber \\  &+& \xi  -  \sum_{i=1}^{d} \frac{1}{2} \big[tr( ( \vect{\sigma}_{k}^{(i)} )^{T} \vect{P} {\hat{\vect{W}}^{(i)}}_{k} ) \nonumber \\ &-& \lambda^{(i)} R({\vect{W}^{(i)}_{k}}).
  \label{eq15}
\end{align} 
On simplification, we get  
\begin{align} 
  H( \hat{\vect{\theta}}_{k} ) - \mathcal{H}( \hat{\vect{\theta}}_{k} ) &=&  \sum_{i=1}^{d}\frac{1}{2} tr( ( \vect{\delta}_{k}^{(i)} - \vect{\sigma}_{k}^{(i)})^{T} \vect{P} {\hat{\vect{W}}^{(i)}}_{k}) \nonumber + \xi  
\end{align} 
This difference can be further simplified as
\begin{align} 
  &&  H( \hat{\vect{\theta}}_{k} )-\mathcal{H}( \hat{\vect{\theta}}_{k} ) =   \sum_{i=1}^{d}  \frac{1}{2} \big[tr( f^{(i-1)}(\vect{x}) \epsilon_{k} \nonumber \\  && (\vect{\mathcal{T}}^{(i)}_k - \vect{B}^{(i)}_k )^T \vect{I}^{(i)}  \vect{P} {\hat{\vect{W}}^{(i)}}_{k} ) + \xi,
\end{align} 
where we use $\vect{B}^{(i)}_{k} = \vect{U}^{(i)}_{k}(\Sigma^{(i)}_{k}+s \times I^{(i)})\vect{V}^{(i)}_{k},$ and  $\vect{\mathcal{T}}^{(i)}_{k} = \vect{U}^{(i)}_{k} \Sigma^{(i)}_{k} \vect{V}^{(i)}_{k}$ to get  
\begin{align} 
    H( \hat{\vect{\theta}}_{k} )-\mathcal{H}( \hat{\vect{\theta}}_{k} ) &=&   \sum_{i=1}^{d}  \frac{1}{2} \big[tr( f^{(i-1)}(\vect{x}) \epsilon_{k} \nonumber \\  (\vect{U}^{(i)}_{k} \Sigma^{(i)}_{k} \vect{V}^{(i)}_{k} &-& \vect{U}^{(i)}_{k}(\Sigma^{(i)}_{k}+s \times I^{(i)})\vect{V}^{(i)}_{k}  ) \vect{I}^{(i)}  \vect{P} {\hat{\vect{W}}^{(i)}}_{k} ) \nonumber \\ &+& \xi,
\end{align} 
which yields
\begin{align} 
    H( \hat{\vect{\theta}}_{k} )-\mathcal{H}( \hat{\vect{\theta}}_{k} ) &=&  \sum_{i=1}^{d}  \frac{1}{2} tr( f^{(i-1)}(\vect{x}) \epsilon_{k} \nonumber \\  ( \vect{U}^{(i)}_{k}(-s &\times& I^{(i)})\vect{V}^{(i)}_{k}  ) \vect{I}^{(i)}  \vect{P} {\hat{\vect{W}}^{(i)}}_{k} ) + \xi.
\end{align} 
We can further simplify the difference as 
\begin{align} 
    && H( \hat{\vect{\theta}}_{k} )-\mathcal{H}( \hat{\vect{\theta}}_{k} ) =  \xi - s \sum_{i=1}^{d} \frac{1}{2} \nonumber \\ && tr( f^{(i-1)}(\vect{x}) \epsilon_{k}  \vect{U}^{(i)}_{k}\vect{I}^{(i)} \vect{V}^{(i)}_{k}  \vect{I}^{(i)}  \vect{P} {\hat{\vect{W}}^{(i)}}_{k} )
\end{align} 
Taking norm both sides to get 
\begin{align} 
    && H( \hat{\vect{\theta}}_{k} )-\mathcal{H}( \hat{\vect{\theta}}_{k} )  =  \| \xi - s \sum_{i=1}^{d} \frac{1}{2} \nonumber \\ && tr( f^{(i-1)}(\vect{x}) \epsilon_{k}  \vect{U}^{(i)}_{k}\vect{V}^{(i)}_{k}  \vect{I}^{(i)}  \vect{P} {\hat{\vect{W}}^{(i)}}_{k} )\| 
\end{align} 
Using the trace property along with triangle inequality to get 
\begin{align} 
    \| H( \hat{\vect{\theta}}_{k} )-\mathcal{H}( \hat{\vect{\theta}}_{k} )\| \leq \|\xi\| + |s| \sum_{i=1}^{d} \frac{1}{2} \nonumber \\  \|f^{(i-1)}(\vect{x})\| \|\epsilon_{k}\| \|\vect{U}^{(i)}_{k}\vect{V}^{(i)}_{k}\| \|\vect{P}\| \|{\hat{\vect{W}}^{(i)}}_{k}\|).
\end{align} 
We invoke the condition $\| {\hat{\vect{W}}^i}_{k}\| \leq W_B,$ $f^{(i-1)}(\vect{x}) <\sqrt{\eta}$, for $i=1,\ldots,d$, and choose $\vect{P}$ such that $\|\vect{P}\| \leq 1,$. Furthermore, $\|\vect{U}^{(i)}_{k}\vect{I}^{(i)}\vect{V}^{(i)}_{k}\|$ is a norm taken over the product of orthonormal matrices and the norm of an orthonormal matrix is 1. Under these conditions, we may write 
\begin{align} 
    \| H( \hat{\vect{\theta}}_{k} )-\mathcal{H}( \hat{\vect{\theta}}_{k} )\| \leq \frac{|s| \sqrt{\eta} d}{2} \|\epsilon_{k}\|  W^{(i)}_B + ||\xi||.
\end{align} 
Taking expected value both sides provides 
\begin{align} 
    E\| H( \hat{\vect{\theta}}_{k} )-\mathcal{H}( \hat{\vect{\theta}}_{k} )\| &\leq& E \big[ \frac{|s| \|\epsilon_{k}\|  \sqrt{\eta} d}{2}  W^{(i)}_B + ||\xi|| \big] \nonumber \\
    &\leq&  \frac{E[|s| \|\epsilon_{k}\|] \sqrt{\eta} d}{2}   W^{(i)}_B + ||\xi||    \\
    &\leq&\frac{  E \big[ |s| \big] E \big[ \|\epsilon_{k}\|] \sqrt{\eta} d}{2}   W^{(i)}_B\big] + ||\xi||. \nonumber
\end{align} 
Since, by assumption 1, $E \big[ |s| \big]= 0$, then $E \| H( \hat{\vect{\theta}}_{k} )-\mathcal{H}( \hat{\vect{\theta}}_{k} )\| \rightarrow  \| \xi \|.$

\end{proof}
In Proposition \ref{prop:cost_error}, we show that the cost defined as part of the EDL algorithm \eqref{total_cost} is an approximation to the original cost \eqref{eq:orignal} with residuals controlled by the choice of the perturbations. In other words, the residual will be zero when the choice of perturbations are uniformly zero.  
In the next theorem, we seek to establish the convergence of $J_E$.  

\begin{thm}
Let Assumption 1 be true. Let the gradient of the empirical cost be given as $\frac{\partial J_{E}( \hat{\vect{\theta}}_{k} ) }{\partial \hat{\vect{\theta}}_{k}} = [ f^{(1)}(\vect{x}) \epsilon_{k} \vect{\mathcal{T}}^{(1)}, \cdots, f^{(d)}(\vect{x}) \epsilon_{k}  \vect{\mathcal{T}}^{(d)}].$ Consider the weight update rule as in \eqref{eq:weight_update_gen} with $\vect{\Delta}_{k} = -\alpha(k) \times \frac{ \partial \mathcal{H}_{k}( \hat{\vect{\theta}}(k))  }{\partial \hat{\vect{\theta}}_{k} }$ with $\alpha(k) > 0$, and  $\frac{\partial \mathcal{H}_{k}( \hat{\vect{\theta}}_{k}) }{\partial \hat{\vect{\theta}}_{k}}  =  [ f^{(1)}(\vect{x}) \epsilon_{k}  \vect{{B}}^{(1)}+ \frac{\partial \lambda^{(i)} R(\hat{\vect{W}}^{(1)})}{\partial\hat{\vect{W}}^{(1)}},   \cdots, f^{(d)}(\vect{x}) \epsilon_{k} \vect{{B}}^{(d)} + \frac{\partial \lambda^{(i)} R(\hat{\vect{W}}^{(d)})}{\partial\hat{\vect{W}}^{(d)}} ],$ where the feedback matrix is defined as in Proposition \ref{prop:cost_error}. If the singular values of $\vect{B}^{(i)}, i = 1, \cdots d$, are non-zero 
then $J_E$ converges asymptotically in the mean.
 \end{thm}
\begin{proof}
We will begin with the first order expansion of $J_{E}(\hat{\vect{\theta}}_{k+1})$ around $\hat{\vect{\theta}}_{k}$ to get 
\begin{align}
        J_{E}(\hat{\vect{\theta}}_{k+1}) &=& J_{E}(\hat{\vect{\theta}}_{k}) + \bigg(\frac{\partial J_{E}(  \hat{\vect{\theta}}_{k} ) }{\partial \hat{\vect{\theta}}_{k}}\bigg)^T [\hat{\vect{\theta}}_{k+1} - \hat{\vect{\theta}}_k],
\end{align}
We add and subtract $J(\hat{\vect{\theta}}_{k}) $ on the right hand side to achieve the first difference as
\begin{align}
     J_{E}(\hat{\vect{\theta}}_{k+1}) - J_{E}(\hat{\vect{\theta}}_{k})  &=& \bigg(\frac{\partial J_{E}( \hat{\vect{\theta}}_{k} ) }{ \partial \hat{\vect{\theta}}_{k}}\bigg)^T [\hat{\vect{\theta}}_{k+1} - \hat{\vect{\theta}}_k],
\end{align}
Substitute the weight updates given by $ -\alpha(k) \frac{\partial \mathcal{H}_{k}( \hat{\vect{\theta}}_{k}) }{\partial \hat{\vect{\theta}}_{k}}$, we get 
\begin{align}
     J_{E}(\hat{\vect{\theta}}_{k+1}) - J_{E}(\hat{\vect{\theta}}_{k})  = - \alpha(k)\bigg[  \bigg(\frac{\partial J_{E}(  \hat{\vect{\theta}}_{k} ) }{ \partial \hat{\vect{\theta}}_{k}}\bigg)^T  \frac{\partial \mathcal{H}_{k}( \hat{\vect{\theta}}_{k}) }{\partial \hat{\vect{\theta}}_{k}} \bigg].
\end{align}
Expanding $\frac{\partial J_{E}( \hat{\vect{\theta}}_{k} ) }{\partial \hat{\vect{\theta}}_{k}} = [ f^{(1)}(\vect{x}) \epsilon_{k} \vect{\mathcal{T}}^{(1)} \cdots f^{(d)}(\vect{x}) \epsilon_{k}  \vect{\mathcal{T}}^{(d)}] $ and $\frac{\partial \mathcal{H}_{k}( \hat{\vect{\theta}}_{k}) }{\partial \hat{\vect{\theta}}_{k}}  = \bigg[ \bigg( f^{(1)}(\vect{x}) \epsilon_{k}  \vect{{B}}^{(1)}+ \frac{\partial \lambda^{(i)} R(\hat{\vect{W}}^{(1)})}{\partial\hat{\vect{W}}^{(1)}} \bigg) \cdots \bigg(f^{(d)}(\vect{x}) \epsilon_{k} \vect{{B}}^{(d)} + \frac{\partial \lambda^{(i)} R(\hat{\vect{W}}^{(d)})}{\partial\hat{\vect{W}}^{(d)}}\bigg) \bigg] $ and write 
\begin{align}
        &&  J_{E}(\hat{\vect{\theta}}_{k+1}) - J_{E}(\hat{\vect{\theta}}_{k}) = -\alpha(k) \sum_{i = 1}^d \bigg[ \bigg( f^{(i-1)}(\vect{x}) \epsilon_{k}  (\vect{\mathcal{T}}^{(i)} \bigg)^{T} \nonumber \\  && \bigg( f^{(i-1)}(\vect{x}) \epsilon_{k}  \vect{{B}}^{(i)}  \vect{P}  + \frac{\partial \lambda^{(i)} R(\hat{\vect{W}}^{(i)})}{\partial \hat{\vect{W}}^{(i)}}  \bigg) \bigg],
\end{align}
which provides 
\begin{align}
        &&  J_{E}(\hat{\vect{\theta}}_{k+1}) - J_{E}(\hat{\vect{\theta}}_{k}) \nonumber \\  &=& -\sum_{i = 1}^d \alpha(k)  \bigg[ \bigg( f^{(i-1)}(\vect{x}) \epsilon_{k} \vect{\mathcal{T}}^{(i)} \bigg)^{T}  \bigg( f^{(i-1)}(\vect{x}) \epsilon_{k}  \vect{{B}}^{(i)}  \vect{P} \bigg) \nonumber \\ &+& \bigg( f^{(i-1)}(\vect{x}) \epsilon_{k} \vect{\mathcal{T}}^{(i)} \bigg)^{T} \bigg( \frac{\partial \lambda^{(i)} R(\hat{\vect{W}}^{(i)})}{\partial \hat{\vect{W}}^{(i)}}  \bigg) \bigg],
\end{align}
Collecting terms to get 
\begin{align}
        &&  J_{E}(\hat{\vect{\theta}}_{k+1}) - J_{E}(\hat{\vect{\theta}}_{k}) \nonumber \\  &=& - \alpha(k) \sum_{i = 1}^d   \bigg[ \vect{\mathcal{T}}^{(i)T} \epsilon^T_{k} f^{(i-1)}(\vect{x})^T f^{(i-1)}(\vect{x}) \epsilon_{k}  \vect{{B}}^{(i)}  \vect{P}  \nonumber \\ &+& \vect{\mathcal{T}}^{(i)T} \epsilon^T_{k} f^{(i-1)}(\vect{x})^T \bigg( \frac{\partial \lambda^{(i)} R(\hat{\vect{W}}^{(i)})}{\partial \hat{\vect{W}}^{(i)}}  \bigg) \bigg],
\end{align}
Substitute  $\vect{B}^{(i)}_{k} = \vect{U}^{(i)}_{k}(\Sigma^{(i)}_{k}+s \times I^{(i)})\vect{V}^{(i)}_{k},$ and  $\vect{\mathcal{T}}^{(i)}_{k} = \vect{U}^{(i)}_{k} \Sigma^{(i)}_{k} \vect{V}^{(i)}_{k}$ to write 
\begin{align}
        &&  J_{E}(\hat{\vect{\theta}}_{k+1}) - J_{E}(\hat{\vect{\theta}}_{k})
        = - \alpha(k)  \sum_{i = 1}^d  \bigg[ [\vect{U}^{(i)}_{k}(\Sigma^{(i)}_{k})\vect{V}^{(i)}_{k}]^T \epsilon^T_{k} \nonumber \\
        && f^{(i-1)}(\vect{x})^T f^{(i-1)}(\vect{x}) \epsilon_{k} [\vect{U}^{(i)}_{k}(\Sigma^{(i)}_{k}+s \times I^{(i)})\vect{V}^{(i)}_{k}]  \vect{P}  \nonumber \\
        &+&\vect{\mathcal{T}}^{(i)T} \epsilon^T_{k} f^{(i-1)}(\vect{x})^T 
        +\bigg( \frac{\partial \lambda^{(i)} R(\hat{\vect{W}}^{(i)})}{\partial \hat{\vect{W}}^{(i)}}  \bigg) \bigg],
\end{align}
which we expand to get 
\begin{align}
        &&  J_{E}(\hat{\vect{\theta}}_{k+1}) - J_{E}(\hat{\vect{\theta}}_{k})= -\alpha(k)  \sum_{i = 1}^d  \bigg[ \vect{V}^{(i)T}_{k} {\Sigma^{(i)T}_{k}} {\vect{U}^{(i)T}_{k}} \epsilon^T_{k} \nonumber \\
        && f^{(i-1)}(\vect{x})^T f^{(i-1)}(\vect{x}) \epsilon_{k}\vect{U}^{(i)}_{k}\Sigma^{(i)}_{k}\vect{V}^{(i)}_{k}  \vect{P} 
        + \vect{V}^{(i)T}_{k} {\Sigma^{(i)T}_{k}}\nonumber \\ &&  {\vect{U}^{(i)T}_{k}} \epsilon^T_{k} f^{(i-1)}(\vect{x})^T  f^{(i-1)}(\vect{x}) \epsilon_{k}  \vect{U}^{(i)}_{k}(s \times I^{(i)})\vect{V}^{(i)}_{k}  \vect{P} + \nonumber \\ && \vect{\mathcal{T}}^{(i)T} \epsilon^T_{k} f^{(i-1)}(\vect{x})^T \bigg( \frac{\partial \lambda^{(i)} R(\hat{\vect{W}}^{(i)})}{\partial \hat{\vect{W}}^{(i)}}  \bigg) \bigg] ,
\end{align}
Let $V^{(i)}_1 =  \vect{V}^{(i)T}_{k} {\Sigma^{(i)T}_{k}} {\vect{U}^{(i)T}_{k}} \epsilon^T_{k}f^{(i-1)}(\vect{x})^T f^{(i-1)}(\vect{x}) \epsilon_{k}$ $\vect{U}^{(i)}_{k}\Sigma^{(i)}_{k})\vect{V}^{(i)}_{k}  \vect{P},$ $V^{(i)}_2 = \vect{V}^{(i)T}_{k} {\Sigma^{(i)T}_{k}} {\vect{U}^{(i)T}_{k}} \epsilon^T_{k}$ $f^{(i-1)}(\vect{x})^T f^{(i-1)}(\vect{x}) \epsilon_{k}  [\vect{U}^{(i)}_{k}(s \times I^{(i)})\vect{V}^{(i)}_{k}]  \vect{P}$ and $V^{(i)}_3 = \vect{\mathcal{T}}^{(i)T} \epsilon^T_{k} f^{(i-1)}(\vect{x})^T \bigg( \frac{\partial \sum_{i = 0}^d \lambda^{(i)} R(\hat{\vect{W}}^{(i)})}{\partial \hat{\vect{\theta}}_{k} } \bigg)$ and write 
\begin{align}
        &&  J_{E}(\hat{\vect{\theta}}_{k+1}) - J_{E}(\hat{\vect{\theta}}_{k}) = - \sum_{i = 1}^d \alpha(k) \bigg[   V^{(i)}_1 + V^{(i)}_2 +V^{(i)}_3\bigg],
\end{align}
For the proof to be complete, it suffices to show that the right hand side is negative which is true as long as the three terms in the brackets $(V^{(i)}_1, V^{(i)}_2, V^{(i)}_3)$ are positive.  The term $V^{(i)}_1$ is  positive as $ \vect{V}^{(i)T}_{k} {\Sigma^{(i)T}_{k}} {\vect{U}^{(i)T}_{k}} \epsilon^T_{k}f^{(i-1)}(\vect{x})^T f^{(i-1)}(\vect{x}) \epsilon_{k}$ $\vect{U}^{(i)}_{k}\Sigma^{(i)}_{k})\vect{V}^{(i)}_{k} $ is a complete square and $\vect{P}$ is positive definite by assumption.  Next, consider the term $V^{(i)}_2$
\begin{align}
      V^{(i)}_2 &=&   \vect{V}^{(i)T}_{k} s{\Sigma^{(i)}_{k}}^T \bigg[ {\vect{U}^{(i)}_{k}}^T \epsilon^T_{k} f^{(i-1)}(\vect{x})^T \nonumber \\ && f^{(i-1)}(\vect{x}) \epsilon_{k}  \vect{U}^{(i)}_{k}\bigg]I^{(i)} \vect{V}^{(i)}_{k}  \vect{P}
\end{align}
The terms in the bracket form a complete square. Therefore, if $\vect{P}$ is positive definite and $s$ is chosen such that $s{\Sigma^{(i)}_{k}}^T$ is positive definite, so $V^{(i)}_2$ is positive. Next, consider the third term $V^{(i)}_3$
\begin{align}
      V^{(i)}_3 &=& \lambda^{(i)} \vect{\mathcal{T}}^{(i)T} \epsilon^T_{k} f^{(i-1)}(\vect{x})^T \bigg( \frac{\partial  R(\hat{\vect{W}}^{(i)})}{\partial \hat{\vect{W}}^{(i)} } \bigg).
\end{align}
Choosing $\lambda^{(i)}  = sign(\vect{\mathcal{T}}^{(i)T} \epsilon^T_{k} f^{(i-1)}(\vect{x})^T \bigg( \frac{\partial  R(\hat{\vect{W}}^{(i)})}{\partial \hat{\vect{W}}^{(i)} } \bigg)) \times c,$where $c \in [0,1]$, we get, $V^{(i)}_3>0.$ 
Next, we write
\begin{align}
         & E[J_{E}(\hat{\vect{\theta}}_{k+1}) - J_{E}(\hat{\vect{\theta}}_{k})] = - \sum_{i = 1}^d \alpha(k) E \bigg[   V^{(i)}_1 + V^{(i)}_2 +V^{(i)}_3\bigg], \nonumber \\
        &= - \sum_{i = 1}^d \alpha(k) \Bigg( E \bigg[   V^{(i)}_1 \bigg]+ E\bigg[V^{(i)}_2\bigg] 
        +E \bigg[V^{(i)}_3\bigg] \Bigg),&
\end{align}
By definition of expected values $E[x] = \int p(x) x $, if $x>0$ and $p(x)>0,$ $E[x] >0.$ As a consequence, for a positive definite function $J_E(\cdot)$ along with $V^{(i)}_1,  V^{(i)}_2$ and $V^{(i)}_3 >0,$ the expected value of the first difference is negative. Therefore, as $k\to \infty$, $E[J_{E}(\hat{\vect{\theta}}_{k+1}) - J_{E}(\hat{\vect{\theta}}_{k})] \rightarrow 0$ and the proof is complete.
\end{proof}
     The results of Proposition 1 reveal that the cost function which is minimized by the EDL is a first-order approximation of the empirical cost, and that the difference between these costs will reduce as the weights of the NN converge. Using Theorem 1, we show that the parameters of the EDL-based update rule can be chosen such the the original cost function asymptotically converges, so that the TD error converges asymptotically.
\subsection{Algorithm}
\begin{figure}[H]
      \centering
      \vspace{-4mm}
    \includegraphics[width=0.9\columnwidth,keepaspectratio]{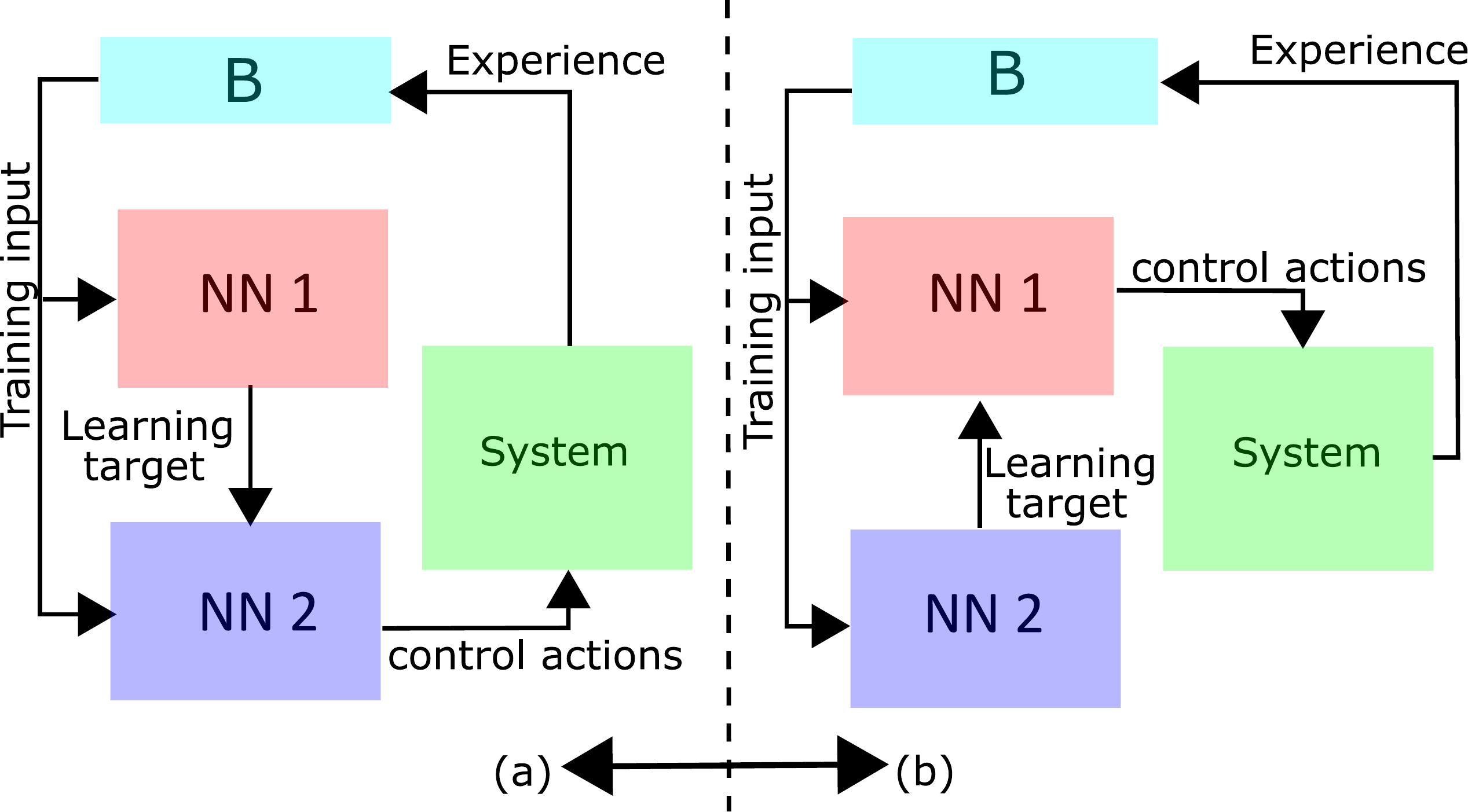}
    \caption{\footnotesize Schematic of the coop DQN architecture. The two NNs~(NN 1 and NN 2) switch roles, refer (a) and (b). In (a) NN 2 is the Q-network that provides actions and learns from the buffer whereas targets are given by NN 1; and in (b) the NNs swap their roles.}
    \label{fig:schematic}
\end{figure}
\begin{figure*}
  \centering
 \includegraphics[width = 0.55\columnwidth]{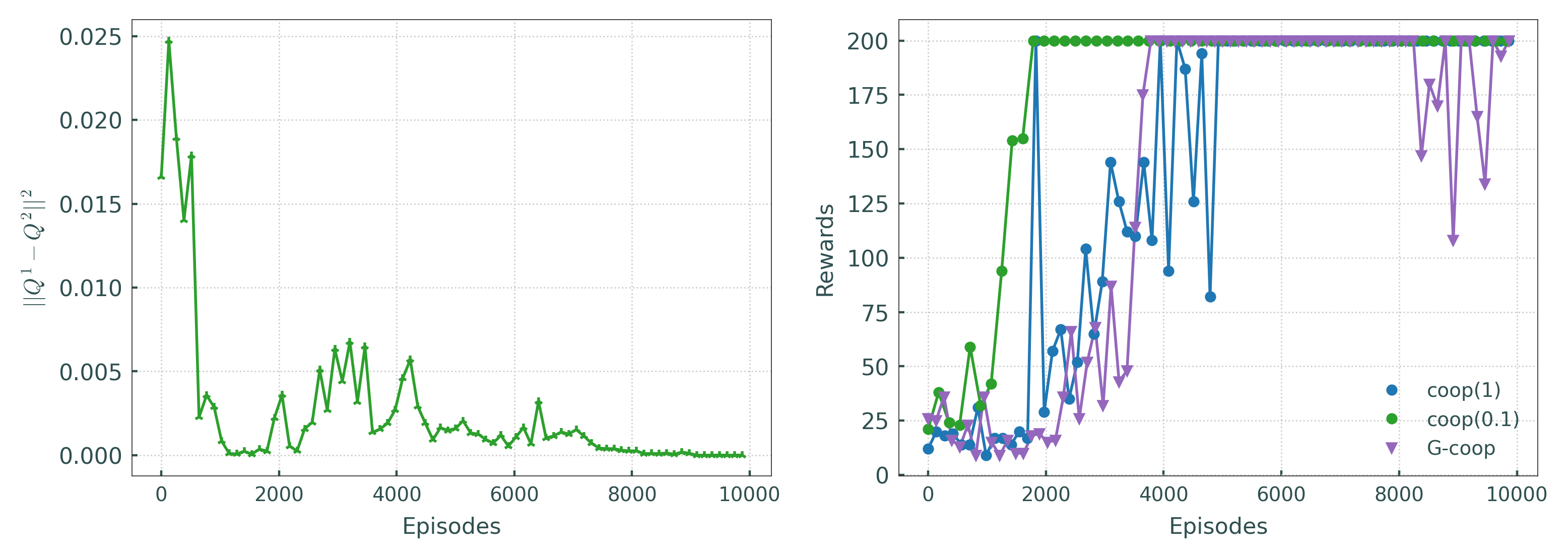}
\caption{\footnotesize Panel A: Mean rewards for different exploration rates, (exploration rate = 0) is the standard SGD, We use ADAM optimizer for this implementation.. Panel B:Difference between the Q-values generated by the two networks as a function of episodes.}
\label{fig:fig_balance}
\end{figure*}

 \begin{figure*}
 \centering
 \includegraphics[width = 0.8\columnwidth,      keepaspectratio]{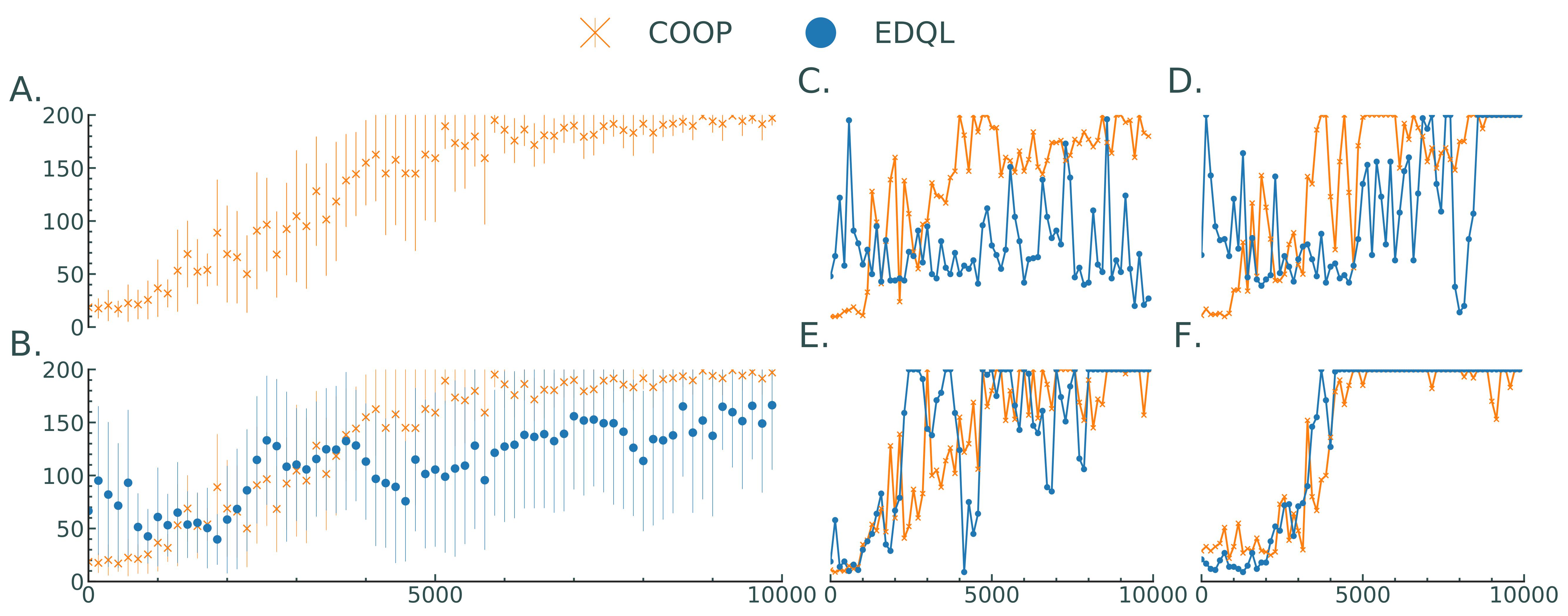}
  \caption{\footnotesize Panel A-B: Mean Rewards (with error bars) depicting variance across different buffer sizes. Buffer sizes considered for this plot are 500,1000, 1500, 2000, 3000, 3500, 4000, 5000). As expected, the average rewards across different buffer sizes are better for the double NN-driven method instead of a single NN driven method. Panel C-F: Trend of rewards with respect to different buffer sizes.} 
  \label{fig:BufferSizes}
\end{figure*}
The training strategy called as cooperative or {\em coop} is described in the following and is illustrated in Fig. \ref{fig:schematic}. The training process is performed for a total of $M$ episodes. At the start of episode $1$, we initialize the two NNs, i.e., $Q_1$ and $Q_2$. In each episode there are $C$ plays. In each play,  we provide a control action to the environment and we receive a tuple comprised of (reward, state, next-state). This tuple and the control action is stored in the experience replay buffer~(denoted as $\mathcal{B}$).  In this training strategy, each episode comprises in repetition of two phases. Phase 1: we first choose $Q_2$ to be target~(the target network) and $Q_1$ to be the network model to be trained~(the actor network); Phase 2: we choose $Q_1$ to be the target network and $Q_2$ to be the the actor network.

We first initiate phase 1, for each play, the actor network is used to gather control action. At the end of each play, the reward and the states are obtained from the environment and $Q_1$ is repeatedly updated using the proposed updated rule. Once, the first phase is completed, we initiate phase 2. In phase 2, $Q_2$ is used to provide control actions and the weights of $Q_2$ are updated. We alternatively switch between phase 1 and phase 2 for every $C$ plays. For each update, we sample a batch of data from the experience replay and use it to evaluate the error and update the weights. This process of learning reaches its culmination when the TD error is small, that is $Q_1$ and $Q_2$ have both reached a common location, where the cost is zero/minimum. At this point, there is no incentive for updating their weights and cooperative strategy has converged. 

In the cooperative learning strategy, the two NNs are utilized to estimate the TD error. Therefore, the asymptotic convergence of $J_E$~(Theorem~1) demonstrates the convergence of the TD error as $J_E$ is squared TD error. 
We evaluate this learning strategy against the DQN benchmark~(In our results, we refer to this as DQL) and the results are summarized in the next section.
\section{Simulation Results}
\label{sec:result}
 We use the  cartpole example~\cite{kumar2020balancing} for illustration of our two NN control strategy. The goal of the controller in this example is to balance the pole on the cart and the control action in this example is whether to move left or right. The states of the system must be extracted from images provided by the open AI simulator that provides the python 3.6 environment for implementation. All simulations are performed using NVIDIA Tesla V100 SXM2 w/32GB HBM2 and NVIDIA Tesla K80 w/dual GPUs. Since the measurements are in the form of images, the two NNs are constructed with two convolutional layers and two feed-forward layers each while \emph{relu} activation functions are utilized. The two NNs are identical with respect to hyper-parameters and network structure. We develop three realizations for comparison: (1) DQL--target NN is a copy of the action NN with gradient driven updates (this is the DQN setup proposed in \cite{mnih2015human}); (2) EDQL--target NN is a copy of the action NN with EDL-based updates; (3) G-coop--two NN setup~(coop strategy) with standard gradient driven updates and (4) coop--two NN setup~(coop strategy) with error driven updates.
 
 The number of outputs in this network are equal to the number of action~(two in the case of cartpole). We record the episodic progression of the average reward~(instantaneous reward averaged over $100$ episodes) and the cumulative rewards~(average output of $Q_{1}$ and $Q_2$). We report the mean and standard deviation from past 100 consecutive recordings.  
We execute our networks for a total of $10000$ episodes and record $Q$ function values and the cumulative reward at each episode. Coop method achieves a score of $200$ with a standard deviation of $0.2123,$ refer Table. \ref{tab:1}. The reward values for cartpole is capped at $200.$ We choose $ C = 50, len(\mathcal{B}) = 5000$ with an exploration rate of $s = 0.05$ decaying exponentially.  

\noindent \textbf{Convergence of the two NNs: } We record the progression of $\|Q_1 -Q_2\|$ with respect to episodes in Fig. \ref{fig:fig_balance}(Panel-A). We see from the figure that $\|Q_1-Q_2\|$ converges to zero near $2000^{th}$ episode at which coop strategy reaches the best cumulative reward values, observed from Fig.\ref{fig:BufferSizes}(Panel-B). These results indicate that the best possible reward for cartpole are obtained when the two NNs converge.
\begin{table}
\centering
\small
\begin{tabular}{ccccc}
\hline 
\hline
Cartpole     &  193.6(0.71)  &  199.6(0.82) & 200(0.82)            & \textbf{200(0.2123)} \\
\hline
\end{tabular}
\caption{\footnotesize Mean score (standard deviation) over 100 episodes at Cartpole. \label{tab:1}, From left to right DQL, EDQL, G-coop, coop    }
\end{table}

\noindent \textbf{Effect of exploration:} We also analyze the affect of different exploration rate for the coop methodology and compare it to that of G-coop~(coop with exploration rate of 0). These results are shown on Panel B of Fig. \ref{fig:fig_balance}. Coop converges quickly when the exploration rate is $0.1$ but exhibits large oscillations when compared with G-coop~(no exploration). Furthermore, note that coop~(0.1) reaches optimal very fast(as fast as the case when the exploration rate of $1$ is used) and is stable. Large exploration rate often provides fast convergence but incurs large oscillation as observed in Panel B of Fig. \ref{fig:fig_balance} and is unstable. It is commonly known that a right balance between exploration and exploitation improves the performance of a reinforcement learning methodology. The behavior of coop with respect to exploration rates further substantiates this idea. Implicit exploration improved convergence and the intelligent choice of exploration rate provides methodical exploration.

\noindent \textbf{Different Buffer Sizes:} We compare coop with EDQL for various buffer sizes in panels C-F. To this end, we choose buffer sizes to be 1000, 2000, 3000 and 4000 where the markers indicate the mean and the lines describe the error band. The best performance for both EDQL and coop can be observed when the buffer size is chosen to be 4000. However, a steady deterioration in performance can be seen while moving from panels F to C~(i.e., when the buffer size is reduced from 4000 to 1000).   However, the deterioration in the performance of EDQL is much larger relative to coop. The results indicate that that the use of the second NN provides convergence benefits. Consider now Panel B where the trend for coop~(this curve is shown in Panel A as well) is consistently above that for EDQL. The use of a second neural network therefore allows us to reduce this memory load by providing good performance for smaller buffer sizes. Thus, the coop strategy is also sample efficient. In these simulations, we have demonstrated that the our approach is sample efficient, provides faster convergence and improves convergence.

\section{Conclusions}\label{sec: conclusions}
In this paper, a two DNN-driven learning-based control strategy is presented to synthesize feedback policies to steer systems generating feedback data in the form of images. Our approach provides the following advantages when compared to the traditional DQN-based RL strategy: (1) the two DNNs reduce the need for large buffer sizes, (2) the coop strategy improves convergence time for  learning the $Q$-network. We also demonstrate the convergence of EDL algorithm. Despite significant improvements in comparison with the DQN, major challenges are still present in the context of employing the proposed approach for synthesizing learning-based online controls to steer the system as desired. For instance, a typical episodic RL setup is considered in this paper, where the RL agent learns iteratively. Secondly, the proposed strategy only considers discrete action space with only finitely many feasible control inputs. 


\bibliographystyle{IEEEtran}
\bibliography{CoRL_Refs}  

\end{document}